\documentclass[letterpaper, 10 pt, conference]{ieeeconf}\IEEEoverridecommandlockouts                     

\usepackage{times}
\usepackage{amsmath}
\usepackage{amssymb}
\usepackage{mathtools}


\usepackage{amsthm}
\usepackage{algorithm}
\usepackage{algpseudocode}
\usepackage{booktabs}
\usepackage{cite} %
\usepackage{xcolor}
\usepackage{tikz}
\definecolor{figprimary}{HTML}{046B99}
\definecolor{figsecondary}{HTML}{C62D1F}
\usepackage{silence}
\usepackage[skip=0pt,font=footnotesize]{caption} %

\newtheorem{theorem}{Theorem}

\newtheorem{corollary}[theorem]{Corollary}
\newtheorem{assumption}[theorem]{Assumption}
\theoremstyle{definition}

\theoremstyle{remark}
\newtheorem{remark}[theorem]{Remark}

\makeatletter\def\@IEEEauthorblockconfadjspace{-1.75em}\makeatother

\title{\LARGE \bf
Full-Order Safety for Nonlinear Reduced-Order MPC
}

\author{Luis Pabon$^{1,\dagger}$, Johannes K\"{o}hler$^{2}$, Revant Kasichainula$^{1}$, and Marco Pavone$^{1}$%
\thanks{*This work was supported by the Office of Naval Research under Grant N00014-23-S-B001.}%
\thanks{$^{1}$Department of Aeronautics and Astronautics, Stanford University, USA}%
\thanks{$^{2}$Department of Mechanical Engineering, Imperial College London, UK}%
\thanks{$^{\dagger}$Corresponding author: {\tt\small lpabon@stanford.edu}}%
}

\begin{document}

\makeatletter
\g@addto@macro\@maketitle{
  \captionsetup{type=figure}\setcounter{figure}{0}
  \centering
  \includegraphics[width=0.9\textwidth]{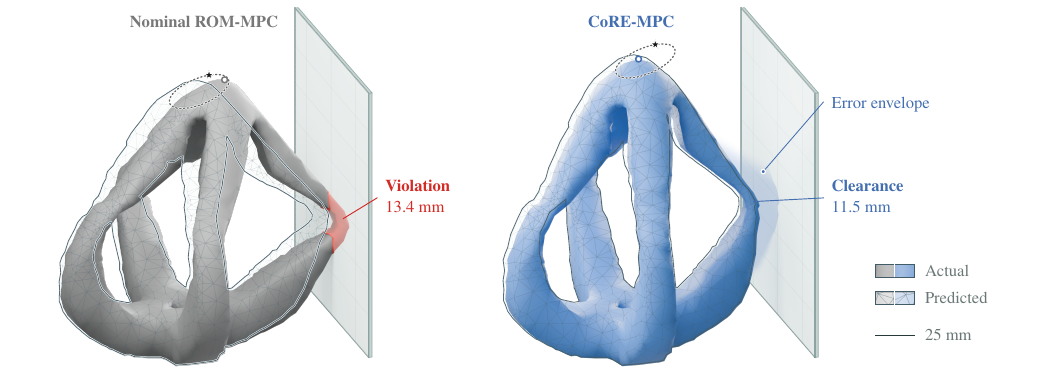}
    \captionof{figure}{\textbf{Full-order constraint satisfaction with a reduced-order model.} The solid surfaces show a $9{,}768$-state finite-element simulation of a soft robot, and the wire shells show a two-state ROM prediction of it at the same instant ($t=1.64\,\mathrm{s}$). With the nominal model, ROM-MPC (gray) predicts $0.1\,\mathrm{mm}$ clearance but actually penetrates the wall by $13.4\,\mathrm{mm}$ (red) as it fails to account for the prediction error of the ROM. By separately bounding the error along the approximation manifold and off it, including their coupling, CoRE-MPC (blue) keeps its predicted error envelope (shaded), and hence every node, clear of the wall with the same ROM, horizon, cost function, and solver.
    The dashed ellipse represents the common XY reference trajectory at the tip height and the star marks the current target along the trajectory.
    \label{fig:hero}}
  \vspace{-3mm}
}
\makeatother
\maketitle
\thispagestyle{empty}
\pagestyle{empty}

\begin{abstract}
Reduced-order models (ROMs) make model predictive control (MPC) of high-dimensional nonlinear systems tractable, but the resulting modeling error can lead to constraint violations on the real system. We introduce CoRE-MPC, a robust MPC formulation that plans with a nonlinear ROM and guarantees recursive feasibility and constraint satisfaction for the full-order system. It can be applied with a broad class of nonlinear model reduction techniques and the resulting MPC does not require evaluating the full-order model online, thus ensuring computational efficiency. 
The key insight is in how the ROM affects the dynamics of the prediction error. We account for this with two error radii that are propagated alongside the ROM and provide a guaranteed bound on the prediction error.
The approach is validated in simulation on an $80$-state oscillator chain, a Burgers equation, and a $9{,}768$-state finite-element soft robot, where a two-state ROM with data-driven bounds keeps all $1{,}628$ nodes clear of a wall at a median solve time below $1\,\mathrm{ms}$ for a $20\,\mathrm{ms}$ control period. Across these examples, CoRE-MPC is $13$ to $61$ times faster than full-order robust MPC at only $27$ to $35\%$ higher closed-loop cost, whereas existing certified schemes are either infeasible or at least twice as conservative.\\
Code: [released upon publication].
\end{abstract}

\section{Introduction}

Reduced-order models (ROMs) enable computationally tractable predictive control for high-dimensional nonlinear systems such as fluid flows, flexible structures, soft robots, and humanoids~\cite{benner2015survey,goury2018mor,dellasantina2023survey,ghansah2025hierarchical}. Safety, however, is a property of the real system, whose trajectory the ROM only approximates, and the two can differ significantly. For example, a ROM can predict a collision-free trajectory while the physical robot still hits an obstacle (Fig.~\ref{fig:hero}). In practice, this mismatch is absorbed by hand-tuned constraint margins or by a larger ROM, which trade performance and computation for safety without guaranteeing it.

Robust MPC accounts for model error by bounding it as a function of the measured state and input and propagating the bound as a tube along the plan~\cite{kohler2021robust}. The prediction error of a ROM is different: it depends on the full state, of which only the reduced coordinates are measured, so its unmeasured part evolves with its own dynamics and feeds back into the reduced prediction.

For linear FOMs, the error dynamics are linear and reduced-order MPC with full-order guarantees is well developed, with uniform~\cite{sopasakis2013rompc,lorenzetti2022linear,tonkens2021soft,wang2022tube} or input-dependent~\cite{lohning2014rompc,kohler2026filtered} error bounds. Treating the nonlinearities as additional bounded disturbances extends these approaches to nonlinear dynamics, which is, however, prohibitively conservative (cf. Section~\ref{sec:chain}). Nonlinear ROMs, such as proper orthogonal decomposition (POD)--Galerkin projections~\cite{holmes1996turbulence,goury2018mor}, quadratic manifolds~\cite{barnett2022quadratic}, and spectral submanifolds~\cite{haller2016ssm},
can provide a more accurate ROM for nonlinear dynamics.
However, using these ROMs directly in an MPC formulation can only guarantee constraints on the ROM prediction, not the real system. Applying a standard nonlinear robust MPC~\cite{kohler2021robust} to ROMs treats the unmodelled dynamics as an independent model mismatch, resulting in overly conservative or even infeasible operation. To the best of our knowledge, the robust nonlinear reduced-order MPC (RN-ROMPC) in~\cite{alora2023rompc} is the only scheme that certifies full-order constraints using only ROM prediction. However, its application requires an invariant manifold, which is available only for specific systems near an equilibrium. Thus, there exists no MPC scheme that can ensure constraint satisfaction using the nonlinear ROMs that are common in high-dimensional soft robots.

\emph{Contribution:} We present CoRE-MPC, a reduced-order MPC that enforces constraints on the full-order system for a general class of nonlinear ROMs. The prediction error within the reduced coordinates is measured, and can thus be reset at every update, but the error outside them is not. CoRE-MPC therefore bounds the two separately, as a tangential and a transverse error, together with their coupling.
\begin{itemize}
\item \textbf{Error bound}: A coupled reduction-error certificate (CoRE, Thm.~\ref{thm:core_certificate}) bounds the prediction error for general ROMs using a two-dimensional comparison system.
\item \textbf{Theoretical guarantees}: Online, CoRE-MPC propagates two error radii with the ROM, at a computational cost comparable to nominal ROM-MPC with two extra states, and guarantees recursive feasibility and constraint satisfaction for the original full-order system (Thm.~\ref{thm:core_mpc}).
\item \textbf{Numerical experiments}: We demonstrate CoRE-MPC on an oscillator chain, a Burgers equation, and the $9{,}768$-state finite-element soft robot of Fig.~\ref{fig:hero}, where nominal ROM-MPC violates the constraints. On the oscillator chain, all assumptions of Theorem~\ref{thm:core_mpc} are satisfied. Compared to full-order robust MPC, CoRE-MPC is $13$ to $61$ times faster at $27$ to $35\%$ higher closed-loop cost. Even when invariant-manifold assumption of RN-ROMPC~\cite{alora2023rompc} holds, CoRE-MPC incurs less than a quarter of its excess closed-loop cost. On the soft robot, where full-order robust MPC is computationally prohibitive, CoRE-MPC keeps all $1{,}628$ nodes clear of the wall at a median solve time below $1$~ms.
\end{itemize}

\textit{Outline:} Section~\ref{sec:problem} states the problem, Section~\ref{sec:core_certificate} derives the error certificate, Section~\ref{sec:core_mpc} presents CoRE-MPC and its closed-loop guarantees, and Section~\ref{sec:experiments} reports the experiments.

\textit{Notation:} $\|\cdot\|$ denotes the Euclidean norm and its induced matrix norm, and $\langle\cdot,\cdot\rangle$ the Euclidean inner product. Inequalities between vectors or matrices hold componentwise. A map is $C^k$ if it is $k$ times continuously differentiable. $D$ denotes the Jacobian and $D_x$ the Jacobian with respect to $x$. $D^+$ denotes the upper right Dini derivative, $D^+v(t):=\limsup_{h\downarrow0}\,[v(t{+}h)-v(t)]/h$, which reduces to $\dot v$ wherever $v$ is differentiable.

\section{Problem Formulation}
\label{sec:problem}

Consider the full-order model (FOM)
\begin{equation}
    \dot x = F(x,u)+d,
    \label{eq:fom}
\end{equation}
with state $x\in\mathbb R^N$, input $u\in\mathbb R^m$, $C^1$ dynamics $F$, and an unknown piecewise continuous disturbance $d$ with $\|d(t)\|\leq\bar d$. We impose the input and state constraints
\begin{equation}
    u\in\mathcal U,\qquad h(x)\leq 0,
    \label{eq:constraints}
\end{equation}
at all times, with a compact set $\mathcal U\subset\mathbb R^m$ and a $C^1$ constraint function $h=(h_1,\ldots,h_{n_h})$. The input constraints encode, for instance, actuator limits, and the state constraints collision avoidance and workspace bounds over the entire body of a soft robot (Fig.~\ref{fig:hero}). Real-time control relies on the reduced-order model (ROM)
\begin{equation}
    \dot z=R(z,u), \qquad \widehat x=\Phi(z),
    \label{eq:rom}
\end{equation}
with reduced state $z\in\mathbb R^r$, $r\ll N$, locally Lipschitz dynamics $R$, and a $C^2$ lifting $\Phi:\mathbb R^r\to\mathbb R^N$ that reconstructs the full state from $z$. Its image $\Phi(\mathbb R^r)$, the states the ROM can represent, is the \emph{approximation manifold}. A matrix $P\in\mathbb R^{r\times N}$ extracts the reduced coordinates $Px$ of the full state and is a left inverse of the lifting,
\begin{equation}
    P\Phi(z)=z.
    \label{eq:left_inverse}
\end{equation}
This setup includes linear projections, $\Phi(z)=Vz$ and $P=W^\top$ with $W^\top V=I$, such as POD--Galerkin~\cite{holmes1996turbulence}, Petrov--Galerkin~\cite{benner2015survey}, and modal truncation, as well as nonlinear liftings such as quadratic manifolds~\cite{barnett2022quadratic} and spectral submanifolds~\cite{haller2016ssm}, with $R$ projected from $F$ or learned from data. Beyond the left inverse~\eqref{eq:left_inverse}, the results below require neither invariance of the approximation manifold nor stability of the ROM, although both influence the growth of the error bounds (Remark~\ref{rem:M}).

Online, the controller measures only the reduced coordinates $Px(t_k)$ at the update times $t_k=k\delta$, $\delta>0$, at which the MPC re-plans. We further assume a known bound $\sigma_0\geq\|x(0)-\Phi(Px(0))\|$ on the initial distance to the approximation manifold ($\sigma_0=0$ if the system starts on it).

The problem is to design a controller that (i) plans with the ROM~\eqref{eq:rom} without evaluating the FOM online, (ii) guarantees the constraints~\eqref{eq:constraints} for the closed-loop FOM under every admissible disturbance, and (iii) minimizes a task cost.

\section{Coupled Reduction-Error Certificates}
\label{sec:core_certificate}

Guaranteeing~\eqref{eq:constraints} for the FOM requires bounding the prediction error $x-\widehat x=x-\Phi(z)$ of the ROM. We do so with radii $\rho=(\rho_\parallel,\rho_\perp)$ for its two components~\eqref{eq:error_definitions}. Section~\ref{sec:exact_dynamics} derives the exact error dynamics, and Section~\ref{sec:core_cert} bounds their growth and coupling rates and constructs the bounds offline.

\subsection{Tangential and Transverse Error Dynamics}
\label{sec:exact_dynamics}
Analogous to RN-ROMPC~\cite{alora2023rompc}, we split the prediction error into two components via the nonlinear projection $\Phi\circ P$ onto the approximation manifold $\Phi(\mathbb R^r)$,
\begin{equation}
    e_\parallel:=Px-z\in\mathbb R^r,
    \qquad
    e_\perp:=x-\Phi(Px)\in\mathbb R^N.
    \label{eq:error_definitions}
\end{equation}
Here the \emph{tangential error} $e_\parallel$ measures prediction error in the reduced coordinates, whereas the \emph{transverse error} $e_\perp$ measures departure from the approximation manifold. By~\eqref{eq:left_inverse},
\begin{equation}
    Pe_\perp=0,
    \qquad
    x=\Phi(z+e_\parallel)+e_\perp.
    \label{eq:exact_state_decomposition}
\end{equation}
As $Pe_\perp=0$, the measurement $Px$ only determines $e_\parallel$ and leaves the high-dimensional $e_\perp$ unmeasured.

The dynamics of the transverse error are given by
\begin{equation}
    \dot e_\perp=\dot x-D\Phi(Px)P\dot x=Q(Px)\dot x,
    \label{eq:Q_def}
\end{equation}
where we define $Q(\zeta):=I-D\Phi(\zeta)P$. Differentiating~\eqref{eq:left_inverse} gives $PD\Phi(\zeta)=I$, so $Q(\zeta)^2=Q(\zeta)$, $\operatorname{range}Q(\zeta)=\ker P$, and $Q(\zeta)D\Phi(\zeta)=0$, so $Q(\zeta)$ is the oblique projection onto $\ker P$ along the tangent space of the manifold. For a linear lifting $\Phi(z)=Vz$, $Q=I-VP$ is constant.

Substituting~\eqref{eq:fom}, \eqref{eq:rom}, and~\eqref{eq:exact_state_decomposition} into $\dot e_\parallel=P\dot x-\dot z$ and $\dot e_\perp=Q(Px)\dot x$ gives the exact coupled dynamics
\begin{align}
    \dot e_\parallel
    &=\underbrace{PF(\Phi(z+e_\parallel)+e_\perp,u)-R(z,u)}_{=:\,g_\parallel(z,u,e_\parallel,e_\perp)}+Pd,
      \label{eq:exact_eparallel}\\
    \dot e_\perp
    &=\underbrace{Q(z+e_\parallel)F(\Phi(z+e_\parallel)+e_\perp,u)}_{=:\,g_\perp(z,u,e_\parallel,e_\perp)}+Q(z+e_\parallel)d.
      \label{eq:exact_eperp}
\end{align}
We call the disturbance-free part $g:=(g_\parallel,g_\perp)$ the \emph{error drift} and its value at zero error the \emph{residual}. The tangential residual $g_\parallel(z,u,0,0)=PF(\Phi(z),u)-R(z,u)$ measures how closely $R$ reproduces the on-manifold flow and vanishes when the ROM is projected as $R(z,u)=PF(\Phi(z),u)$. The transverse residual $g_\perp(z,u,0,0)$ is the rate at which the FOM flow leaves the manifold.

\subsection{Comparison System for the Error Norms}
\label{sec:core_cert}

For radii $\rho=(\rho_\parallel,\rho_\perp)\geq0$, the errors within the radii that satisfy $Pe_\perp=0$ form the compact convex set\footnote{\label{fn:norm}For simplicity of exposition, we consider the Euclidean norm. In practice, a weighted norm $\|\Theta^{1/2}v\|$, $\Theta\succ0$, is in general required to ensure that stable error dynamics yield negative rates $\mu_\parallel,\mu_\perp$ in~\eqref{eq:rates}~\cite{soderlind2006lognorm}, see Remark~\ref{rem:M}. This corresponds to a linear change of coordinates of $z$ and $x$, under which the left inverse~\eqref{eq:left_inverse} holds for the transformed $\Phi$ and $P$, and hence the analysis applies without modification.}
\begin{equation}
    \mathcal E_{\rho}:=\{(e_\parallel,e_\perp):Pe_\perp=0,\|e_\parallel\|\leq\rho_\parallel,\|e_\perp\|\leq\rho_\perp\}.
    \label{eq:error_set}
\end{equation}
Since $F$ is $C^1$ and $\Phi$ is $C^2$, the drift $g$ is $C^1$ in the errors, hence Lipschitz and one-sided Lipschitz on every compact convex set of errors~\cite[Lemma~1]{aminzare2014contraction}. 
We compute these bounds on user-chosen compact sets $\mathcal E_{\bar\rho}$ with a fixed $\bar\rho=(\bar\rho_\parallel,\bar\rho_\perp)>0$ and a compact region $\mathcal Z\subset\mathbb R^r$ of reduced states containing those the task visits, trading domain of validity against conservatism. 
In particular, at each $(z,u)\in\mathcal Z\times\mathcal U$, there exist continuous one-sided Lipschitz bounds $\mu_\parallel(z,u),\mu_\perp(z,u)$ and Lipschitz bounds $\gamma_{\parallel\perp}(z,u),\gamma_{\perp\parallel}(z,u)\geq0$ such that, for every $(e_\parallel,e_\perp)\in\mathcal E_{\bar\rho}$,
\begin{subequations}
\label{eq:bounds}
\begin{equation}
\begin{aligned}
    \langle e_\parallel,\,
    g_\parallel(z,u,e_\parallel,0)-g_\parallel(z,u,0,0)\rangle
    &\leq \mu_\parallel\|e_\parallel\|^2,\\
    \|g_\perp(z,u,e_\parallel,0)-g_\perp(z,u,0,0)\|
    &\leq \gamma_{\perp\parallel}\|e_\parallel\|,\\
    \|g_\parallel(z,u,e_\parallel,e_\perp)-g_\parallel(z,u,e_\parallel,0)\|
    &\leq \gamma_{\parallel\perp}\|e_\perp\|,\\
    \langle e_\perp,\,
    g_\perp(z,u,e_\parallel,e_\perp)-g_\perp(z,u,e_\parallel,0)\rangle
    &\leq \mu_\perp\|e_\perp\|^2.
\end{aligned}
\label{eq:rates}
\end{equation}
Furthermore, let the continuous bounds $b_\parallel(z,u),b_\perp(z,u)\geq0$ satisfy, for every $\|e_\parallel\|\leq\bar\rho_\parallel$,
\begin{equation}
\begin{aligned}
    b_\parallel(z,u)&\geq\|g_\parallel(z,u,0,0)\|+\|P\|\,\bar d,\\
    b_\perp(z,u)&\geq\|g_\perp(z,u,0,0)\|+\|Q(z+e_\parallel)\|\,\bar d.
\end{aligned}
\label{eq:b_def}
\end{equation}
\end{subequations}

The four rates in~\eqref{eq:rates} assemble into the matrix
\begin{equation}
    M(z,u):=
    \begin{bmatrix}
        \mu_\parallel(z,u) & \gamma_{\parallel\perp}(z,u)\\
        \gamma_{\perp\parallel}(z,u) & \mu_\perp(z,u)
    \end{bmatrix}.
    \label{eq:M_def}
\end{equation}
The growth rates $\mu_\parallel,\mu_\perp$ on the diagonal are negative if the respective error dynamics contract. The transverse error perturbs the reduced prediction through $\gamma_{\parallel\perp}$, and the tangential error drives the state off the manifold through $\gamma_{\perp\parallel}$. The map $b$ bounds the forcings that persist at zero error: the residual $g(z,u,0,0)$ and the disturbance $d$.

For a planned ROM trajectory $(z,u)$, we predict the radii $\rho$ alongside the ROM with the comparison system
\begin{equation}
    \dot\rho=M(z,u)\rho+b(z,u).
    \label{eq:rho_dynamics}
\end{equation}
\begin{theorem}[Coupled reduction-error certificate (CoRE)]
\label{thm:core_certificate}
Let $M$ and $b$ satisfy~\eqref{eq:bounds}, let $(z,u)$ be a planned ROM trajectory--input pair with values in $\mathcal Z\times\mathcal U$ on $[t_0,t_1]$, and let $x$ solve the FOM~\eqref{eq:fom} on $[t_0,t_1]$ under the same input $u$ and any disturbance with $\|d(t)\|\leq\bar d$, with the errors $e_\parallel,e_\perp$ defined by~\eqref{eq:error_definitions}. Let $\rho$ solve the comparison system~\eqref{eq:rho_dynamics} with $\rho_\parallel(t_0)\geq\|e_\parallel(t_0)\|$ and $\rho_\perp(t_0)\geq\|e_\perp(t_0)\|$, and suppose that $\rho<\bar\rho$ on $[t_0,t_1]$. Then $\|e_\parallel\|\leq\rho_\parallel$ and $\|e_\perp\|\leq\rho_\perp$ on $[t_0,t_1]$.
\end{theorem}

\begin{proof}
Throughout, we omit the time argument and the $(z,u)$-dependence. We first show that, for $(e_\parallel,e_\perp)\in\mathcal E_{\bar\rho}$,
\begin{equation}
    D^+\!\begin{bmatrix}\|e_\parallel\|\\\|e_\perp\|\end{bmatrix}
    \leq M\begin{bmatrix}\|e_\parallel\|\\\|e_\perp\|\end{bmatrix}+b,
    \label{eq:core_bound_split}
\end{equation}
with $D^+$ acting componentwise. For $e_\parallel\neq0$, \eqref{eq:exact_eparallel} and the upper Dini derivative of the Euclidean norm give
\begin{align*}
D^+\|e_\parallel\|
&=\frac{e_\parallel^\top}{\|e_\parallel\|}
\bigl[g_\parallel(e_\parallel,e_\perp)+Pd\bigr]\\
&\leq
\frac{e_\parallel^\top}{\|e_\parallel\|}
\bigl[g_\parallel(e_\parallel,0)-g_\parallel(0,0)\bigr] +
\|g_\parallel(0,0)\|\\
&\quad+
\|g_\parallel(e_\parallel,e_\perp)-g_\parallel(e_\parallel,0)\|+\|Pd\|\\
&\leq
\mu_\parallel\|e_\parallel\|
+\gamma_{\parallel\perp}\|e_\perp\|
+b_\parallel,
\end{align*}
where the first inequality follows from telescoping $g_\parallel(e_\parallel,e_\perp)$ around $(e_\parallel,0)$ and $(0,0)$ and from Cauchy--Schwarz (except the first term, which keeps its sign) and the second from~\eqref{eq:bounds} and $\|d\|\leq\bar d$.
The case $e_\parallel=0$ follows from $D^+\|e_\parallel\|\leq\|\dot e_\parallel\|$. A similar argument applied to~\eqref{eq:exact_eperp} gives the second row of~\eqref{eq:core_bound_split}. 
Since $M$ is Metzler, the right-hand side of~\eqref{eq:rho_dynamics} is quasimonotone. Hence, the comparison theorem for systems of differential inequalities~\cite[Cor.~1.7.1]{lakshmikantham1969differential} gives $\|e_\parallel\|\leq\rho_\parallel$ and $\|e_\perp\|\leq\rho_\perp$ as long as $(e_\parallel,e_\perp)\in\mathcal E_{\bar\rho}$, which $Pe_\perp=0$ from~\eqref{eq:exact_state_decomposition}, these bounds, and $\rho<\bar\rho$ ensure on $[t_0,t_1]$ by a continuation argument.
\end{proof}

\begin{corollary}[Radius-dependent bounds]
\label{cor:radius}
Let $M(z,u,\rho)$ and $b(z,u,\rho)$ be continuous, locally Lipschitz and componentwise nondecreasing in $\rho$, and satisfy~\eqref{eq:bounds} on $\mathcal E_\rho$ for every $\rho\geq0$. Then Theorem~\ref{thm:core_certificate} holds for $\dot\rho=M(z,u,\rho)\rho+b(z,u,\rho)$ without the condition $\rho<\bar\rho$.
\end{corollary}
\begin{proof}
Let $v:=(\|e_\parallel\|,\|e_\perp\|)$. Since $(e_\parallel,e_\perp)\in\mathcal E_v$, the derivation of~\eqref{eq:core_bound_split} with the bounds evaluated at $\rho=v$ gives $D^+v\leq G(t,v):=M(z,u,v)v+b(z,u,v)$ at every $t$. For $0\leq w\leq w'$ with $w_i=w_i'$, the monotonicity of $M$ and $b$ and $\gamma\geq0$ give $G_i(t,w)\leq G_i(t,w')$, so $G$ is quasimonotone nondecreasing and locally Lipschitz in $w$, and the comparison theorem~\cite[Cor.~1.7.1]{lakshmikantham1969differential} yields $v\leq\rho$ on $[t_0,t_1]$.
\end{proof}

\emph{Offline construction:} Any offline bounding technique can be used to get~\eqref{eq:bounds}, e.g., composing Lipschitz constants of the FOM nonlinearities~\cite{alora2023rompc}. Sharper values follow from the Jacobian of the drift with respect to the errors. Abbreviating $\zeta:=z+e_\parallel$, $x_e:=\Phi(\zeta)+e_\perp$, $F:=F(x_e,u)$, $F_x:=D_xF(x_e,u)$, $Q:=Q(\zeta)$, and $D\Phi:=D\Phi(\zeta)$, its blocks at any $(z,u,e_\parallel,e_\perp)$, ordered as in~\eqref{eq:M_def}, are
\begin{equation*}
    \begin{bmatrix}
        D_{e_\parallel}g_\parallel & D_{e_\perp}g_\parallel\\
        D_{e_\parallel}g_\perp & D_{e_\perp}g_\perp
    \end{bmatrix}
    =
    \begin{bmatrix}
        PF_xD\Phi & PF_x\\
        QF_xD\Phi-D^2\Phi[PF] & QF_x
    \end{bmatrix},
\end{equation*}
where $D^2\Phi[v]$ denotes the derivative of $D\Phi$ at $\zeta$ in direction $v$. On the convex set $\mathcal E_{\bar\rho}$, one may take $\gamma_{\perp\parallel},\gamma_{\parallel\perp}$ as the suprema of the norms of the off-diagonal blocks and $\mu_\parallel,\mu_\perp$ as the suprema of the logarithmic norms $\mu_2(A):=\lambda_{\max}((A{+}A^\top)/2)$ of the diagonal blocks~\cite[Lemma~1]{aminzare2014contraction}. The blocks with respect to $e_\parallel$ are evaluated at $e_\perp=0$, as in the first two inequalities of~\eqref{eq:rates}. Those with respect to $e_\perp$ act on $\ker P$ only.

These suprema can be upper-bounded in many ways. We use three, one in each example of Section~\ref{sec:experiments}. For polynomial dynamics with a linear lift, closed-form bounds in $(z,u,\rho)$ hold on the current tube (Corollary~\ref{cor:radius}, Section~\ref{sec:burgers}). Otherwise, we fit upper envelopes to the block norms at samples over $\mathcal Z\times\mathcal U\times\mathcal E_{\bar\rho}$ (Section~\ref{sec:chain}).\footnote{\label{fn:cover}Lipschitz margins covering the gaps between samples make the envelopes valid on the whole sampling domain~\cite{bobiti2018sampling} (cf. Section~\ref{sec:chain}).} When the FOM is too large to sample the block Jacobian, we directly fit $M$ and $b$ to prediction errors recorded along closed-loop trajectories~\cite{alora2023rompc}. The bounds~\eqref{eq:bounds} are then checked on the data rather than certified (Section~\ref{sec:soft_robot}).

\begin{remark}[Stability and structure of $M$]\label{rem:M}
While Theorem~\ref{thm:core_certificate} applies to general matrices $M$ constructed according to~\eqref{eq:bounds}, the stability of the growth of the radii depends on whether $M$ admits a Hurwitz constant bound. For $M(z,u)\leq \bar M$ and $b(z,u)\leq\bar b$ on $\mathcal Z\times\mathcal U$, $\bar M$ is Hurwitz iff $\bar\mu_\parallel,\bar\mu_\perp<0$ and $\bar\mu_\parallel\bar\mu_\perp>\bar\gamma_{\parallel\perp}\bar\gamma_{\perp\parallel}$. When $\bar M$ is Hurwitz, $\rho$ is ultimately bounded by $-\bar M^{-1}\bar b$. Negative $\bar\mu_\parallel,\bar\mu_\perp$ require stable error dynamics and a suitable norm or change of coordinates (footnote~\ref{fn:norm}).
As we see in the numerical example (Sec.~\ref{sec:burgers}), 
CoRE-MPC can also be applied when $M$ is not Hurwitz, since re-centering on the measurement can reset $\rho_\parallel$ at every update. The carried transverse radius then stays bounded if, over one sampling interval, its own contraction outweighs the growth it receives through the coupling (thus, a negative $\mu_\perp$ is necessary but not sufficient). 
For invariant manifolds~\cite{haller2016ssm}, $Q(\zeta)F(\Phi(\zeta),0)=0$, the autonomous flow never leaves the manifold, so $\gamma_{\perp\parallel}(z,0)=0$ and $M(z,0)$ is upper triangular. Thus, $M(z,0)$ is Hurwitz iff $\mu_\parallel,\mu_\perp<0$. This is the setting of RN-ROMPC~\cite{alora2023rompc}. In contrast, for general POD projections (Section~\ref{sec:burgers}), the general structure of $M$ proposed in CoRE-MPC is required.
\end{remark}

\section{CoRE-MPC: Predictive Control with Coupled Error Certificates}
\label{sec:core_mpc}

We now use the radii of Theorem~\ref{thm:core_certificate} to tighten the
FOM constraints inside a reduced-order MPC and establish its closed-loop guarantees (Theorem~\ref{thm:core_mpc}).

\emph{Constraint tightening:} By~\eqref{eq:exact_state_decomposition}, $x=\Phi(z+e_\parallel)+e_\perp$ with $(e_\parallel,e_\perp)\in\mathcal E_\rho$ whenever the errors lie within the radii $\rho$. As in Section~\ref{sec:core_cert}, let $c_{j}^{\parallel}(z),c_{j}^{\perp}(z)\geq0$ satisfy, for every $z\in\mathcal Z$ and $(e_\parallel,e_\perp)\in\mathcal E_{\bar\rho}$,
\begin{equation}
\begin{aligned}
    h_j(\Phi(z+e_\parallel))-h_j(\Phi(z))
    &\leq c_{j}^{\parallel}(z)\|e_\parallel\|,\\
    h_j(\Phi(z+e_\parallel)+e_\perp)-h_j(\Phi(z+e_\parallel))
    &\leq c_{j}^{\perp}(z)\|e_\perp\|,
\end{aligned}
\label{eq:tightening_constants}
\end{equation}
e.g., bounds over $\mathcal E_{\bar\rho}$ on $\|Dh_j\,D\Phi\|$ and on $\|Dh_j\|_{\ker P}$, the maximum of $\|Dh_j\,v\|$ over unit $v\in\ker P$. Adding the two lines, $(e_\parallel,e_\perp)\in\mathcal E_\rho$ implies $h_j(x)\leq\widehat h_j(z,\rho)$ for the tightened constraint
\begin{equation}
    \widehat h_j(z,\rho)
    :=h_j(\Phi(z))
      +c_{j}^{\parallel}(z)\rho_\parallel
      +c_{j}^{\perp}(z)\rho_\perp,
    \label{eq:affine_tightening}
\end{equation}
which CoRE-MPC enforces in place of $h_j\leq0$. For affine constraints $h_j=a_j^\top x-\bar h_j$, $c_{j}^{\perp}=\|a_j^\top\|_{\ker P}$ is exact.

Let $T\geq\delta$ be the prediction horizon. Given nonnegative stage and terminal costs $\ell$ and $V_f$ and writing
$\chi:=(z,\rho)$ for the augmented prediction state, CoRE-MPC solves at each
$t_k$ the following optimization problem:
\begin{subequations}
\label{eq:core_mpc_ocp}
\begin{align}
&V_T(Px(t_k),\sigma_k):=\min_{z,\rho,u}\int_0^T\!\ell(z(\tau),u(\tau))\,\mathrm d\tau+V_f(z(T))
\label{eq:core_mpc_cost}\\
&\text{s.t. }\ \dot z(\tau)=R(z(\tau),u(\tau)),\quad\forall\tau\in[0,T],
\label{eq:core_mpc_dyn}\\
&\hphantom{\text{s.t. }\ }\dot\rho(\tau)=M(z(\tau),u(\tau))\rho(\tau)+b(z(\tau),u(\tau)),
\label{eq:core_mpc_rhodyn}\\
&\hphantom{\text{s.t. }\ }\|Px(t_k)-z(0)\|\leq\rho_\parallel(0),\quad
  \rho_\perp(0)=\sigma_k,
\label{eq:core_mpc_initial}\\
&\hphantom{\text{s.t. }\ }\widehat h_j(z(\tau),\rho(\tau))\leq0,\quad j=1,\ldots,n_h,
\label{eq:core_mpc_tightening}\\
&\hphantom{\text{s.t. }\ }(z(\tau),u(\tau))\in\mathcal Z\times\mathcal U,\quad
  0\leq\rho(\tau)<\bar\rho,
\label{eq:core_mpc_domain}\\
&\hphantom{\text{s.t. }\ }\chi(T)\in\mathcal X_f.
\label{eq:core_mpc_terminal}
\end{align}
\end{subequations}
Here the initial condition~\eqref{eq:core_mpc_initial} leaves the nominal
state $z(0)$ free within the initial tube, as is standard in tube MPC. In particular, it admits any $z(0)$ between the measurement $Px(t_k)$, with $\rho_\parallel(0)=0$, and the shifted previous plan. The constraint~\eqref{eq:core_mpc_domain}
keeps the plan inside the domain
$\mathcal Z\times\mathcal U$ of the bounds~\eqref{eq:bounds},
the scalar $\sigma_k\geq\|x(t_k)-\Phi(Px(t_k))\|$ is a transverse-error bound
maintained by Algorithm~\ref{alg:core}, and the terminal set $\mathcal X_f$ is specified in
Assumption~\ref{ass:core_terminal}. Let $(z_k^\star,\rho_k^\star,u_k^\star)$
denote a minimizer of~\eqref{eq:core_mpc_ocp}, assumed to exist. Algorithm~\ref{alg:core}
summarizes offline design and online operation.

\begin{algorithm}[H]
\caption{CoRE-MPC}
\label{alg:core}
\begin{algorithmic}[1]
\Statex \textbf{Offline}
\State Choose the ROM $(P,\Phi,R)$ satisfying~\eqref{eq:left_inverse}; fix
$\bar\rho$.
\State Compute the bounds~\eqref{eq:bounds} (Section~\ref{sec:core_cert}).
\State Construct the tightened constraints $\widehat h_j$ via~\eqref{eq:affine_tightening}.
\State Design terminal ingredients $(\mathcal X_f,\kappa_f)$
(Ass.~\ref{ass:core_terminal}) and a terminal cost $V_f$.
\Statex \textbf{Online}, from $k=0$
\State Measure $Px(t_k)$; solve~\eqref{eq:core_mpc_ocp} for
$(z_k^\star,\rho_k^\star,u_k^\star)$.
\State Apply $u(t_k+\tau)=u_k^\star(\tau)$, $\tau\in[0,\delta)$.
\State Set $\sigma_{k+1}=\rho_{\perp,k}^\star(\delta)$, $k\leftarrow k+1$;
go to step~5.
\end{algorithmic}
\end{algorithm}

Compared to a nominal ROM-MPC scheme, \eqref{eq:core_mpc_ocp} adds the two
scalar states $\rho$ and replaces the state constraints by the precomputed
tightened constraints $\widehat h_j$, so its computational demand is comparable to that of a nominal
MPC for an $(r{+}2)$-dimensional system and the FOM vector field $F$ is never evaluated
online. Through $M(z,u)$ and $b(z,u)$, the tube grows where the local reduction
geometry is unfavorable, so the optimizer trades performance against
tube growth.

\emph{Closed-loop guarantees:} The following assumption captures the
required terminal ingredients for the augmented state $\chi$, whose dynamics we write as $\dot\chi=f_{\rm C}(\chi,u):=(R(z,u),\,M(z,u)\rho+b(z,u))$.

\begin{assumption}[Terminal ingredients]
\label{ass:core_terminal}
There exist a closed terminal set
$\mathcal X_f\subset\mathcal Z\times[0,\bar\rho)$ and a continuous terminal law
$\kappa_f:\mathcal X_f\to\mathcal U$ such that, for every $\chi=(z,\rho)\in\mathcal X_f$,
\begin{align}
&\widehat h_j(z,\rho)\leq0,
\quad j=1,\ldots,n_h,
\label{eq:terminal_constraints}\\
&\mathcal X_f\text{ is positively invariant under }
  \dot\chi=f_{\rm C}(\chi,\kappa_f(\chi)).
\label{eq:terminal_invariance}
\end{align}
\end{assumption}

These are standard MPC terminal conditions (admissibility,
viability)~\cite{chen1998quasi}, posed for the augmented
state $\chi=(z,\rho)$ as in~\cite{kohler2021robust}.
Viability now requires the terminal region to tolerate the growth of
$\rho$, which stability of the ROM alone does not ensure, since nothing is
reset within a plan. Under a constant Hurwitz bound $\bar M$, $\bar b$ as
in Remark~\ref{rem:M}, valid on $\mathcal X_f$, the box $[0,\rho_{\rm f}]$ with
$\bar M\rho_{\rm f}+\bar b\leq0$, $\rho_{\rm f}<\bar\rho$, is positively invariant
for~\eqref{eq:core_mpc_rhodyn} and serves as the radius component of a
product terminal set. Online, \eqref{eq:core_mpc_rhodyn} retains the state-
and input-dependent $M(z,u)$ and $b(z,u)$.

\begin{theorem}[CoRE-MPC closed-loop guarantees]
\label{thm:core_mpc}
Let $M$ and $b$ satisfy~\eqref{eq:bounds}, let $c_j^\parallel,c_j^\perp$ satisfy~\eqref{eq:tightening_constants}, let Assumption~\ref{ass:core_terminal} hold,
let $\sigma_0\geq\|e_\perp(0)\|$, and suppose that the CoRE-MPC problem~\eqref{eq:core_mpc_ocp} is feasible at $t_0=0$. Then, under
Algorithm~\ref{alg:core}, \eqref{eq:core_mpc_ocp} is feasible at every
$t_k$ and the closed-loop FOM satisfies
\begin{equation}
    u(t)\in\mathcal U,\qquad h(x(t))\leq0,\qquad t\geq0.
    \label{eq:fom_constraint_guarantee}
\end{equation}
\end{theorem}

\begin{proof}
The argument is that of~\cite[Thm.~1]{alora2023rompc}, with the coupled
comparison system~\eqref{eq:rho_dynamics} in place of two one-way coupled scalar
tubes. \emph{Constraint satisfaction:}
If~\eqref{eq:core_mpc_ocp} is feasible at $t_k$ with
$\sigma_k\geq\|e_\perp(t_k)\|$, then the initial-tube constraint~\eqref{eq:core_mpc_initial} gives
$(e_\parallel(t_k),e_\perp(t_k))\in\mathcal E_{\rho_k^\star(0)}$, the domain constraint~\eqref{eq:core_mpc_domain} keeps the
certificate valid, and Theorem~\ref{thm:core_certificate} gives
$(e_\parallel(t_k+\tau),e_\perp(t_k+\tau))\in\mathcal E_{\rho_k^\star(\tau)}$ for $\tau\in[0,\delta]$.
The tightened constraints~\eqref{eq:core_mpc_tightening} then imply
$h(x)\leq0$ there by~\eqref{eq:affine_tightening}, and the input constraint in~\eqref{eq:core_mpc_domain} gives $u\in\mathcal U$. \emph{Recursive feasibility:} The same
containment at $\tau=\delta$ is the initial
condition~\eqref{eq:core_mpc_initial} for the shifted plan
$(\chi_k^\star(\cdot+\delta),u_k^\star(\cdot+\delta))$ at $t_{k+1}$, and
appending $\kappa_f$ on $[T-\delta,T]$ keeps $\chi$ in $\mathcal X_f$
by invariance~\eqref{eq:terminal_invariance}. Induction over $k$ completes the proof.
\end{proof}

If, in addition, $\ell$ is positive definite in $(z,u)$ and $V_f$ is $C^1$ with $DV_f(z)\,R(z,\kappa_f(\chi))\leq-\ell(z,\kappa_f(\chi))$ on $\mathcal X_f$, the value-decrease argument of~\cite{chen1998quasi} applies to the shifted candidate solution of the proof, so the planned nominal state $z_k^\star(0)$ converges to the origin. The measured state $Px(t_k)$ stays within $\bar\rho_\parallel$ of it by the initial-tube constraint~\eqref{eq:core_mpc_initial}.

\begin{figure*}[t]
\centering
\includegraphics[width=\textwidth]{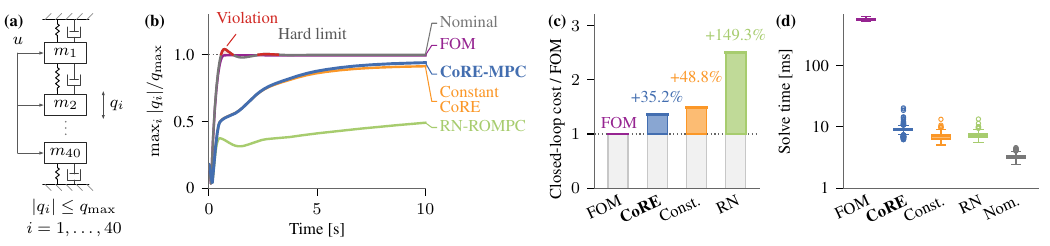}
\caption{Oscillator chain. FOM-MPC (purple), CoRE-MPC (blue), Constant CoRE (orange), RN-ROMPC (green), and nominal ROM-MPC (gray). (a) Forty masses driven by one actuator with the shape of the first mode; every displacement is constrained. (b) Largest normalized displacement over all masses in the case with the largest nominal violation. (c) Median closed-loop cost over the $48$ cases relative to FOM-MPC; nominal ROM-MPC is excluded because it violates the constraint. (d) Solve times over $495$ warm updates per controller from five cases run serially in balanced order on one thread. Boxes show quartiles, lines medians, whiskers the most extreme samples within $1.5$ interquartile ranges, and circles outliers.}
\label{fig:chain}
\end{figure*}

\section{Experiments}
\label{sec:experiments}

We evaluate CoRE-MPC on three systems of increasing dimension:
\begin{itemize}
\item the oscillator chain (Section~\ref{sec:chain}) quantifies the trade-off between full-order safety, closed-loop cost, and computation on a reduction where the invariant-manifold assumption of RN-ROMPC~\cite{alora2023rompc} holds (Remark~\ref{rem:M}).
\item the Burgers equation (Section~\ref{sec:burgers}) tests the certificate on a POD reduction whose manifold is not invariant, with closed-form bounds that hold for every state and radius.
\item the finite-element soft robot (Section~\ref{sec:soft_robot}) tests whole-body constraint satisfaction, on all $1{,}628$ nodes of a $9{,}768$-state robot, with a two-state ROM and bounds fitted from data.
\end{itemize}
We compare the following controllers:
\begin{itemize}
\item \emph{CoRE-MPC}: the proposed scheme~\eqref{eq:core_mpc_ocp}.
\item \emph{Constant CoRE}: CoRE-MPC with $M$ and $b$ replaced by constants, their domain maxima on the oscillator chain and separately fitted constants on the soft robot.
\item \emph{Nominal ROM-MPC}: the same ROM, horizon, and cost function without accounting for model mismatch.
\item \emph{FOM-MPC}~\cite{kohler2021robust}: robust MPC on the FOM with a single tube propagated by a sampled logarithmic-norm bound and re-centered on the measured full state at every update. As it uses the full model and full state, it serves as the benchmark for closed-loop cost and computation time.
\item \emph{RN-ROMPC}~\cite{alora2023rompc}: the special case of~\eqref{eq:rho_dynamics} with constant $M$ and $\gamma_{\perp\parallel}=0$ by invariance of the manifold (Remark~\ref{rem:M}), and an input-dependent forcing $b(u)=b_0+b_u\|u\|$ with constant gains, which absorbs the input-induced departure from the manifold. Its bounds are sampled from the model on the oscillator chain and fitted from rollouts on the soft robot.
\item \emph{Linear ROMPC}~\cite{lorenzetti2022linear,tonkens2021soft}: MPC on the linearized ROM with the nonlinearity bounded as a disturbance.
\end{itemize}
Within each example, all controllers share the task, horizon, and cost function, and use a weighted Euclidean norm (footnote~\ref{fn:norm}). FOM-MPC measures the full state, whereas the reduced-order controllers measure only $Px$. All controllers use a quadratic terminal cost. On the oscillator chain, CoRE-MPC also imposes the terminal set~\eqref{eq:core_mpc_terminal} (Section~\ref{sec:chain}), so Theorem~\ref{thm:core_mpc} applies. The baselines and the remaining examples omit it for ease of implementation, which relaxes recursive feasibility and so does not disadvantage the baselines. Constraint satisfaction on each sampling interval still follows from Theorem~\ref{thm:core_certificate} whenever the problem is feasible and the bounds hold. Solve times are measured using the SQP solver in acados~\cite{verschueren2022acados} with HPIPM on an Apple M3 Max. All controllers shown fully converge at every control update, so no fallback to the shifted previous plan of the proof of Theorem~\ref{thm:core_mpc} was needed. Implementation details and certificate parameters are documented with the accompanying code.

\begin{figure*}[t]
\centering
\includegraphics[width=\textwidth]{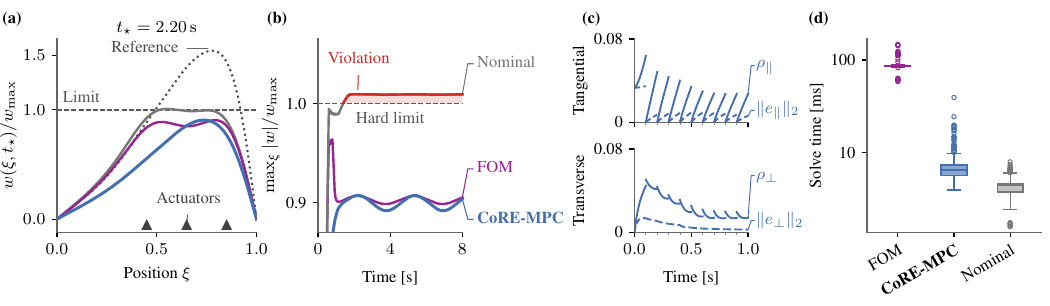}
\caption{Burgers equation. Nominal ROM-MPC (gray), FOM-MPC (purple), and CoRE-MPC (blue). (a) Field at $t^*=2.2$~s with the reference (dotted), the limit (dashed), and the actuator locations (triangles). (b) Largest normalized field value over all nodes. (c) CoRE-MPC radii (solid) and realized errors (dashed) over the first second; the tangential radius, chosen by the optimizer, restarts at zero after the first update, whereas the transverse radius is carried over. (d) Solve times over $316$ updates per controller from four cases, with box conventions as in Fig.~\ref{fig:chain}.}
\label{fig:burgers}
\end{figure*}

\subsection{Oscillator Chain}
\label{sec:chain}

\emph{Setup:} We consider the oscillator chain from~\cite{jain2022invariant} with $40$ masses connected by nonlinear springs and linear dampers, resulting in $N=80$ states (Fig.~\ref{fig:chain}(a)). A single actuator with the spatial shape of the first linear mode drives the chain, and a bounded force disturbance acts on the last mass. Every displacement is constrained, $|q_i|\leq q_{\max}$, and the reference requests a first-mode peak one third above this limit, so the constraint is active. The ROM has $r=2$ states on a spectral submanifold~\cite{haller2016ssm}, which is invariant for the autonomous FOM, so the assumption of RN-ROMPC holds. The horizon is $2$~s and the sampling time $100$~ms. The bounds of CoRE-MPC and RN-ROMPC are obtained from samples on the same domain. Those of CoRE-MPC include the covering margins of footnote~\ref{fn:cover}, so~\eqref{eq:bounds} holds on all of $\mathcal Z\times\mathcal U\times\mathcal E_{\bar\rho}$. CoRE-MPC further imposes the terminal set~\eqref{eq:core_mpc_terminal}. The augmented terminal region uses a continuous-time, two-radius adaptation of the construction in~\cite[Prop.~6]{kohler2021adaptive}. Its numerical definition and its admissibility and invariance certificate, \eqref{eq:terminal_constraints} and~\eqref{eq:terminal_invariance}, are provided with the accompanying code. We test $48$ initial conditions for $10$~s each, $16$ of which were generated to stress the constraint handling of the controllers.

\emph{Results:} Fig.~\ref{fig:chain} shows the case with the largest nominal violation and the statistics over all cases. Nominal ROM-MPC violates the displacement limit in $23$ of $48$ cases, by up to $4\%$ of $q_{\max}$ (Fig.~\ref{fig:chain}(b)). In contrast, all robust controllers satisfy the constraints in all cases, and the realized errors remain inside the predicted tubes throughout, in accordance with Theorem~\ref{thm:core_certificate}. The excess closed-loop cost over FOM-MPC (Fig.~\ref{fig:chain}(c)) increases with the conservatism of the bounds, from $35\%$ for CoRE-MPC to $49\%$ for Constant CoRE and $149\%$ for RN-ROMPC. Constant CoRE and RN-ROMPC replace the state- and input-dependent bounds by constants, so their tubes grow at the worst-case rate along the entire plan. Consequently, RN-ROMPC requires a large conservative backoff from the constraints, while CoRE-MPC can operate close to them (Fig.~\ref{fig:chain}(b)). Linear ROMPC is infeasible in all $48$ cases, since bounding the nonlinear spring force as a disturbance is too conservative. In computation (Fig.~\ref{fig:chain}(d)), CoRE-MPC solves $61$ times faster than FOM-MPC at $2.9$ times the solve time of nominal ROM-MPC, with a 95th percentile of $10.6$~ms against the $100$~ms sampling time. Constant CoRE and RN-ROMPC, which also propagate two radii but omit the terminal set and use simpler bounds, solve in about $80\%$ of the time of CoRE-MPC. Thus, even where the assumption of RN-ROMPC holds, CoRE-MPC incurs less than a quarter of its excess closed-loop cost at comparable computation.

\begin{figure*}[t]
\centering
\includegraphics[width=\textwidth]{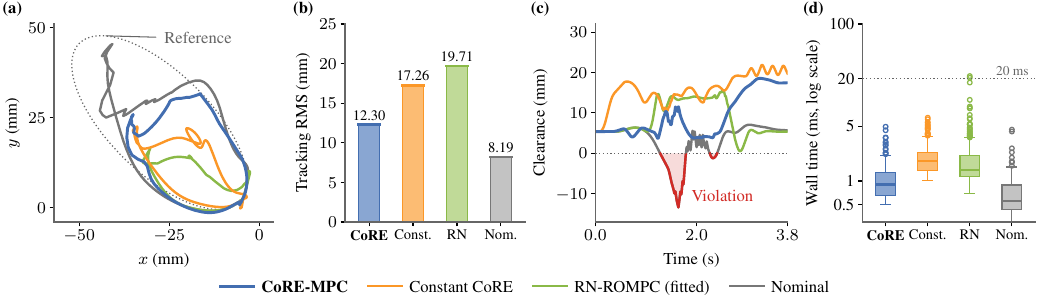}
\caption{Soft robot. CoRE-MPC (blue), Constant CoRE (orange), RN-ROMPC (green), and nominal ROM-MPC (gray). (a) Tip paths and the reference (dotted). (b) RMS tip tracking error over one complete reference traversal ($0.3$ to $3.3$~s). (c) Minimum clearance over all $1{,}628$ nodes; red marks wall penetration. (d) Controller wall times over $190$ updates, including the check of all nodes, with box conventions as in Fig.~\ref{fig:chain}; times exclude FEM integration.}
\label{fig:fem}
\end{figure*}

\subsection{Burgers Equation}
\label{sec:burgers}

\emph{Setup:} We consider the viscous Burgers equation on the unit interval with viscosity $\nu=0.06$, discretized on $128$ nodes and driven by three Gaussian actuators (Fig.~\ref{fig:burgers}(a)). A six-mode POD--Galerkin ROM~\cite{kunisch1999burgers} captures $99.998\%$ of the snapshot energy, and its approximation manifold is not invariant. The magnitude of the field is bounded by $w_{\max}$ at every node, while the reference peaks $54\%$ above this limit. The horizon is $0.3$~s and the sampling time $100$~ms. A bounded, spatially localized sinusoidal disturbance acts on the field, and its $24$ equally spaced phases define the $24$ cases, each run for $8$~s from rest. The FOM dynamics are quadratic in the state and the lift is linear, so the Jacobian blocks and residuals are polynomial in $(z,u,e)$ and the norm inequalities give closed-form upper bounds satisfying~\eqref{eq:bounds} on $\mathcal E_\rho$ for every $\rho\geq0$ instead of on a fixed $\mathcal E_{\bar\rho}$, affine in $\rho$ with nonnegative coefficients. The comparison system~\eqref{eq:rho_dynamics} becomes $\dot\rho=M(z,u,\rho)\rho+b(z,u,\rho)$, to which Theorem~\ref{thm:core_certificate} extends by Corollary~\ref{cor:radius}.

\emph{Results:} Fig.~\ref{fig:burgers} shows the case with the largest nominal violation and the solve times. Despite the small discarded snapshot energy, nominal ROM-MPC violates the limit in all $24$ cases, by up to $0.9\%$ of $w_{\max}$ and for most of the run (Fig.~\ref{fig:burgers}(b)). CoRE-MPC and FOM-MPC complete all cases without violation and converge to a similar steady state at about $0.90\,w_{\max}$. Along the reference, both couplings and the transverse residual are nonzero, so the coupled structure of~\eqref{eq:rho_dynamics} is required, and $M$ is not Hurwitz. CoRE-MPC nevertheless remains feasible in every case: with the free initialization~\eqref{eq:core_mpc_initial}, the optimizer selects $\rho_\parallel(0)=0$ at every update after the first, so the tangential radius restarts, and the carried transverse radius stays bounded (the gain from $\sigma_k$ to $\sigma_{k+1}$ is at most $0.73$ over all intervals, cf. Remark~\ref{rem:M}). The certificate is valid and tight: every tube contains the realized errors, which reach $63\%$ of the transverse and $98\%$ of the tangential radius within the applied intervals over all cases (Fig.~\ref{fig:burgers}(c)), and along the evaluated trajectories the coupling bounds exceed the exact Jacobian block norms by factors of at most $1.15$ and $1.38$. CoRE-MPC incurs $27\%$ higher closed-loop cost than FOM-MPC and solves $13$ times faster, at $1.5$ times the solve time of nominal ROM-MPC, with a 95th percentile of $12.3$~ms against the $100$~ms sampling time (Fig.~\ref{fig:burgers}(d)).

\subsection{Finite-Element Soft Robot}
\label{sec:soft_robot}

\emph{Setup:} We consider the cable-driven diamond robot of~\cite{goury2018mor}, simulated in SOFA~\cite{faure2012sofa} with $9{,}768$ states, $1{,}628$ nodes, and four cable forces between $0$ and $3$~N. The task is to track an elliptical tip reference that extends beyond a wall (Fig.~\ref{fig:fem}(a)), so it cannot be tracked exactly on its upper arc, while every node must stay clear of the wall. The ROM has $r=2$ states on a spectral submanifold of the FEM~\cite{haller2016ssm}, whose lifting reconstructs all nodes. The horizon is $60$~ms and the sampling time $20$~ms. Each controller runs the task for $3.8$~s without disturbance. Since the FOM is too large to sample the block Jacobian, an input-dependent $M$ and a state- and input-dependent $b$ are fitted to the prediction errors of ten earlier rollouts on the same robot and wall with nearby elliptical references (Section~\ref{sec:core_cert}), and the bounds of RN-ROMPC and Constant CoRE are fitted to the same data, so all three are empirical. The weighting of the error norm is optimized on the residuals of the calibration rollouts to minimize the wall tightening over all nodes (footnote~\ref{fn:norm}). Only the nodes with the largest predicted wall envelope enter the optimization, and every plan is checked against all $1{,}628$ nodes, which CoRE-MPC never violates.

\emph{Results:} The clearance in Fig.~\ref{fig:fem}(c) summarizes the outcome: nominal ROM-MPC penetrates the wall by $13$~mm, whereas the three robust controllers keep every node clear, with minimum clearances of $3.7$~mm for CoRE-MPC, $5.4$~mm for Constant CoRE, and $0.5$~mm for RN-ROMPC. RN-ROMPC~\cite{alora2023rompc}, with its constant upper-triangular $M$, backs away from the wall along the upper arc yet comes closest to it. Averaged over matched prediction stages with identical initial radii, the wall tightening of RN-ROMPC is $2.3$ times that of CoRE-MPC ($15.3$ against $6.6$~mm). CoRE-MPC, whose $M$ and $b$ follow the plan, tracks with $29\%$ less error than Constant CoRE and $38\%$ less than RN-ROMPC (Fig.~\ref{fig:fem}(b)), and the lower error of nominal ROM-MPC is obtained on a trajectory that penetrates the wall. In computation (Fig.~\ref{fig:fem}(d)), all four controllers solve in under $2$~ms at the median, CoRE-MPC in $0.9$~ms against $0.55$~ms for nominal ROM-MPC and never above $5$~ms, whereas two RN-ROMPC updates exceed the $20$~ms period. Full-order MPC is excluded, since one FEM horizon prediction alone takes $0.56$~s, or $28$ sampling periods.

\emph{Validation:} To test the fitted bounds beyond the calibration references, we fit $M$ and $b$ of the three robust controllers above to the prediction errors of $202$ rollouts on $24$ elliptical references with a $10\%$ margin over the recorded errors. We then apply CoRE-MPC and RN-ROMPC on five held-out elliptical references, one of which is repeated with an initial transverse error of $0.5$~mm and again, as a stress test, with a $50$~mN force toward the wall exceeding $\bar d$. In all $14$ held-out runs, every node stays clear of the wall and both radii contain the realized errors at every update, reaching at most $90\%$ of the tangential and $77\%$ of the transverse radius. On the task of Fig.~\ref{fig:fem}, the larger calibration set increases the RMS tracking error of all controllers. CoRE-MPC increases to $17.4$~mm, which is still $20\%$ less than Constant CoRE ($21.7$~mm) and $50\%$ less than RN-ROMPC ($34.7$~mm).

\section{Conclusion}
\label{sec:conclusion}

We have presented CoRE-MPC, a reduced-order MPC that enforces full-order constraints for a general class of nonlinear ROMs, including non-invariant approximation manifolds and ROMs learned from data. The error within the reduced coordinates can be reset at every update, but the error outside them cannot, so CoRE-MPC bounds the two separately and propagates both radii with the ROM at the cost of two extra states. In the numerical studies, this yields a fraction of the excess cost of existing certified schemes where they apply, and real-time whole-body collision avoidance for a soft robot where full-order robust MPC is prohibitive. Open issues include the certification of comparably tight bounds at the finite-element scale, where the bounds are currently fitted to data, and additional sensing of the transverse error, which is propagated open-loop.

\section*{Acknowledgment}
GPT-6 Astra (OpenAI) was used to assist in writing the scripts for Figs.~\ref{fig:hero}--\ref{fig:fem}, and Claude Fable 5.1 (Anthropic) was used for editing and grammar. The authors independently verified every derivation, implementation, and result.

\bibliographystyle{IEEEtran}
\bibliography{references}

@book{holmes1996turbulence,
  author    = {Holmes, Philip and Lumley, John L. and Berkooz, Gal},
  title     = {Turbulence, Coherent Structures, Dynamical Systems and Symmetry},
  publisher = {Cambridge University Press},
  address   = {Cambridge, UK},
  year      = {1996}
}

@article{benner2015survey,
  author  = {Benner, Peter and Gugercin, Serkan and Willcox, Karen},
  title   = {A survey of projection-based model reduction methods for parametric dynamical systems},
  journal = {SIAM Review},
  volume  = {57},
  number  = {4},
  pages   = {483--531},
  year    = {2015},
  doi     = {10.1137/130932715}
}

@article{barnett2022quadratic,
  author  = {Barnett, Joshua and Farhat, Charbel},
  title   = {Quadratic approximation manifold for mitigating the {K}olmogorov barrier in nonlinear projection-based model order reduction},
  journal = {Journal of Computational Physics},
  volume  = {464},
  year    = {2022},
  note    = {{Art.} no. 111348},
  doi     = {10.1016/j.jcp.2022.111348}
}

@inproceedings{alora2023rompc,
  author    = {Alora, John Irvin and Pabon, Luis A. and K{\"o}hler, Johannes and Cenedese, Mattia and Schmerling, Ed and Zeilinger, Melanie N. and Haller, George and Pavone, Marco},
  title     = {Robust nonlinear reduced-order model predictive control},
  booktitle = {2023 62nd IEEE Conference on Decision and Control (CDC)},
  pages     = {4798--4805},
  year      = {2023},
  doi       = {10.1109/CDC49753.2023.10383243}
}

@article{haller2016ssm,
  author  = {Haller, George and Ponsioen, Sten},
  title   = {Nonlinear normal modes and spectral submanifolds: existence, uniqueness and use in model reduction},
  journal = {Nonlinear Dynamics},
  volume  = {86},
  number  = {3},
  pages   = {1493--1534},
  year    = {2016},
  doi     = {10.1007/s11071-016-2974-z}
}

@book{lakshmikantham1969differential,
  author    = {Lakshmikantham, Vangipuram and Leela, Srinivasa},
  title     = {Differential and Integral Inequalities: Theory and Applications, Vol.~{I}: Ordinary Differential Equations},
  series    = {Mathematics in Science and Engineering},
  volume    = {55},
  publisher = {Academic Press},
  address   = {New York},
  year      = {1969}
}

@article{soderlind2006lognorm,
  author  = {S{\"o}derlind, Gustaf},
  title   = {The logarithmic norm. {H}istory and modern theory},
  journal = {BIT Numerical Mathematics},
  volume  = {46},
  number  = {3},
  pages   = {631--652},
  year    = {2006},
  doi     = {10.1007/s10543-006-0069-9}
}

@inproceedings{aminzare2014contraction,
  author    = {Aminzare, Zahra and Sontag, Eduardo D.},
  title     = {Contraction methods for nonlinear systems: A brief introduction and some open problems},
  booktitle = {2014 53rd IEEE Conference on Decision and Control (CDC)},
  pages     = {3835--3847},
  year      = {2014},
  doi       = {10.1109/CDC.2014.7039986}
}

@article{bobiti2018sampling,
  author  = {Bobiti, Ruxandra and Lazar, Mircea},
  title   = {Automated-sampling-based stability verification and {DOA} estimation for nonlinear systems},
  journal = {IEEE Transactions on Automatic Control},
  volume  = {63},
  number  = {11},
  pages   = {3659--3674},
  year    = {2018},
  doi     = {10.1109/TAC.2018.2797196}
}

@article{chen1998quasi,
  author  = {Chen, Hong and Allg{\"o}wer, Frank},
  title   = {A quasi-infinite horizon nonlinear model predictive control scheme with guaranteed stability},
  journal = {Automatica},
  volume  = {34},
  number  = {10},
  pages   = {1205--1217},
  year    = {1998},
  doi     = {10.1016/S0005-1098(98)00073-9}
}

@article{kohler2021robust,
  author  = {K{\"o}hler, Johannes and Soloperto, Raffaele and M{\"u}ller, Matthias A. and Allg{\"o}wer, Frank},
  title   = {A computationally efficient robust model predictive control framework for uncertain nonlinear systems},
  journal = {IEEE Transactions on Automatic Control},
  volume  = {66},
  number  = {2},
  pages   = {794--801},
  year    = {2021},
  doi     = {10.1109/TAC.2020.2982585}
}

@article{kohler2021adaptive,
  author  = {K{\"o}hler, Johannes and K{\"o}tting, Peter and Soloperto, Raffaele and Allg{\"o}wer, Frank and M{\"u}ller, Matthias A.},
  title   = {A robust adaptive model predictive control framework for nonlinear uncertain systems},
  journal = {International Journal of Robust and Nonlinear Control},
  volume  = {31},
  number  = {18},
  pages   = {8725--8749},
  year    = {2021},
  doi     = {10.1002/rnc.5147}
}

@article{kohler2026filtered,
  author  = {K{\"o}hler, Johannes and Scholz, Carlo and Zeilinger, Melanie N.},
  title   = {Robust reduced-order model predictive control using peak-to-peak analysis of filtered signals},
  journal = {European Journal of Control},
  year    = {2026},
  note    = {{Art.} no. 101569},
  doi     = {10.1016/j.ejcon.2026.101569}
}

@article{verschueren2022acados,
  title     = {acados---a modular open-source framework for fast embedded optimal control},
  author    = {Verschueren, Robin and Frison, Gianluca and Kouzoupis, Dimitris and Frey, Jonathan and van Duijkeren, Niels and Zanelli, Andrea and Novoselnik, Branimir and Albin, Thivaharan and Quirynen, Rien and Diehl, Moritz},
  journal   = {Mathematical Programming Computation},
  volume    = {14},
  number    = {1},
  pages     = {147--183},
  year      = {2022}
}

@article{lorenzetti2022linear,
  title   = {Linear reduced-order model predictive control},
  author  = {Lorenzetti, Joseph and McClellan, Andrew and Farhat, Charbel and Pavone, Marco},
  journal = {IEEE Transactions on Automatic Control},
  volume  = {67},
  number  = {11},
  pages   = {5980--5995},
  year    = {2022},
  doi     = {10.1109/TAC.2022.3179539}
}

@incollection{faure2012sofa,
  author    = {Faure, Fran{\c{c}}ois and Duriez, Christian and Delingette, Herv{\'e} and Allard, J{\'e}r{\'e}mie and Gilles, Benjamin and Marchesseau, St{\'e}phanie and Talbot, Hugo and Courtecuisse, Hadrien and Bousquet, Guillaume and Peterlik, Igor and Cotin, St{\'e}phane},
  title     = {{SOFA}: A multi-model framework for interactive physical simulation},
  booktitle = {Soft Tissue Biomechanical Modeling for Computer Assisted Surgery},
  series    = {Studies in Mechanobiology, Tissue Engineering and Biomaterials},
  publisher = {Springer},
  address   = {Berlin, Heidelberg},
  pages     = {283--321},
  year      = {2012},
  doi       = {10.1007/8415_2012_125}
}

@article{kunisch1999burgers,
  author  = {Kunisch, K. and Volkwein, S.},
  title   = {Control of the {B}urgers equation by a reduced-order approach using proper orthogonal decomposition},
  journal = {Journal of Optimization Theory and Applications},
  volume  = {102},
  number  = {2},
  pages   = {345--371},
  year    = {1999},
  doi     = {10.1023/A:1021732508059}
}

@article{goury2018mor,
  author  = {Goury, Olivier and Duriez, Christian},
  title   = {Fast, generic, and reliable control and simulation of soft robots using model order reduction},
  journal = {IEEE Transactions on Robotics},
  volume  = {34},
  number  = {6},
  pages   = {1565--1576},
  year    = {2018},
  doi     = {10.1109/TRO.2018.2861900}
}

@article{dellasantina2023survey,
  author  = {Della Santina, Cosimo and Duriez, Christian and Rus, Daniela},
  title   = {Model-based control of soft robots: A survey of the state of the art and open challenges},
  journal = {IEEE Control Systems Magazine},
  volume  = {43},
  number  = {3},
  pages   = {30--65},
  year    = {2023},
  doi     = {10.1109/MCS.2023.3253419}
}

@article{lohning2014rompc,
  author  = {L{\"o}hning, Martin and Reble, Marcus and Hasenauer, Jan and Yu, Shuyou and Allg{\"o}wer, Frank},
  title   = {Model predictive control using reduced order models: Guaranteed stability for constrained linear systems},
  journal = {Journal of Process Control},
  volume  = {24},
  number  = {11},
  pages   = {1647--1659},
  year    = {2014},
  doi     = {10.1016/j.jprocont.2014.07.006}
}

@inproceedings{sopasakis2013rompc,
  author    = {Sopasakis, Pantelis and Bernardini, Daniele and Bemporad, Alberto},
  title     = {Constrained model predictive control based on reduced-order models},
  booktitle = {2013 52nd IEEE Conference on Decision and Control (CDC)},
  pages     = {7071--7076},
  year      = {2013}
}

@article{wang2022tube,
  author  = {Wang, Wei and Koeln, Justin P.},
  title   = {Tube-based robust {MPC} for two-timescale systems using reduced-order models},
  journal = {IEEE Control Systems Letters},
  volume  = {7},
  pages   = {799--804},
  year    = {2023},
  doi     = {10.1109/LCSYS.2022.3226893}
}

@inproceedings{tonkens2021soft,
  author    = {Tonkens, Sander and Lorenzetti, Joseph and Pavone, Marco},
  title     = {Soft robot optimal control via reduced order finite element models},
  booktitle = {2021 IEEE International Conference on Robotics and Automation (ICRA)},
  pages     = {12010--12016},
  year      = {2021}
}

@inproceedings{ghansah2025hierarchical,
  author    = {Ghansah, Adrian B. and Esteban, Sergio A. and Ames, Aaron D.},
  title     = {Hierarchical reduced-order model predictive control for robust locomotion on humanoid robots},
  booktitle = {2025 IEEE-RAS 24th International Conference on Humanoid Robots (Humanoids)},
  pages     = {1--8},
  year      = {2025}
}

@article{jain2022invariant,
  author  = {Jain, Shobhit and Haller, George},
  title   = {How to compute invariant manifolds and their reduced dynamics in high-dimensional finite element models},
  journal = {Nonlinear Dynamics},
  volume  = {107},
  number  = {2},
  pages   = {1417--1450},
  year    = {2022},
  doi     = {10.1007/s11071-021-06957-4}
}
\end{document}